\newcommand{\onlyicml}[1]{}

\documentclass{article}

\usepackage{microtype}
\usepackage{graphicx}
\usepackage{hyperref}
\usepackage{subfigure}
\usepackage{booktabs} %
\usepackage{siunitx}

\usepackage{hyperref}

\usepackage[accepted]{icml2019}

\usepackage{amsmath,amsfonts,amssymb,amsthm}
\usepackage{color}

\newcommand{\E}{\mathbb{E}}
\newcommand{\EE}[1]{\E\!\left[{#1}\right]}
\newcommand{\Ee}[2]{\E_{{#1}}\!\left[{#2}\right]}
\newcommand{\norm}[1]{\left\lVert{#1}\right\rVert}
\newcommand{\abs}[1]{\left\lvert{#1}\right\rvert}

\DeclareMathOperator\spn{span}

\definecolor{darkgreen}{rgb}{0,0.4,0.0}

\newcommand{\removed}[1]{}
\usepackage{soul}

\usepackage[italic,defaultmathsizes]{mathastext}
\usepackage[capitalize]{cleveref}

\newtheorem{theorem}{Theorem}
\newtheorem{lemma}{Lemma}

\newcommand{\DD}{\mathcal{D}}
\newcommand{\tDD}{\mathcal{\tilde{D}}}

\newcommand{\allw}{\mathbf{w}}
\newcommand{\tF}{\mathbf{F}}

\newcommand{\wopt}{w_{\star}}

\icmltitlerunning{Semi-Cyclic Stochastic Gradient Descent}

\begin{document}

\twocolumn[
\icmltitle{Semi-Cyclic Stochastic Gradient Descent}

\icmlsetsymbol{equal}{*}

\begin{icmlauthorlist}
\icmlauthor{Hubert Eichner}{Google}
\icmlauthor{Tomer Koren}{Google}
\icmlauthor{H. Brendan McMahan}{Google}
\icmlauthor{Nathan Srebro}{TTIC}
\icmlauthor{Kunal Talwar}{Google}
\end{icmlauthorlist}

\icmlaffiliation{Google}{Google.}
\icmlaffiliation{TTIC}{Toyota Technological Institute at Chicago.  Part of this work was done while NS was visiting Google.  Authors are listed alphabetically}
\icmlcorrespondingauthor{\onlyicml{NS}}{\texttt{huberte}, \texttt{mcmahan}, \texttt{kunal}, \texttt{tkoren@google.com},\, \texttt{nati@ttic.edu}}

\onlyicml{\icmlkeywords{Machine Learning, ICML}}

\vskip 0.3in
]

\printAffiliationsAndNotice{}  %

\begin{abstract}
We consider convex SGD updates with a block-cyclic structure, i.e.~where each
cycle consists of a small number of blocks, each with many samples
from a possibly different, block-specific, distribution.  This situation arises, e.g., in Federated Learning where the mobile devices available for updates at different times during the day have different characteristics. We show that such block-cyclic structure can significantly deteriorate the
performance of SGD, but propose a simple  approach that allows prediction with the same performance guarantees as for i.i.d., non-cyclic, sampling.
\end{abstract}

\section{Introduction}\label{sec:intro}

When using Stochastic Gradient Descent (SGD), it is important that
samples are drawn at random.  In contrast, cycling over a specific
permutation of data points (e.g.~sorting data by label, or even
alphabetically or by time of collection) can be detrimental  to the
performance of SGD.  But in some applications, such as in Federated
Learning \citep{FL3, FL4} or when running SGD in ``real time'', some cyclic patterns in
the data samples are hard to avoid.  In this paper we investigate how
such cyclic structure hurts the effectiveness of SGD, and how this
can be corrected.  We model this semi-cyclicity through a ``block-cyclic''
structure, where during the course of training, we cycle over blocks
in a fixed order and samples for each block are drawn from a block-specific distribution.

Our primary motivation is Federated Learning. In this setting, mobile devices collaborate in the training of a shared model while keeping the training data decentralized. Devices communicate updates (e.g. gradients) to a coordinating server, which aggregates the updates and applies them to the global model. In each iteration (or round) of Federated SGD, typically a few hundred devices are chosen randomly by the server to participate; critically, however, only devices that are idle, charging, and on a free wireless connection are selected~\cite{FL_BLOG,bonawitz19sysml}. This ensures Federated Learning does not impact the user's experience when using their phone, but also can produce significant diurnal variation in the devices available, since devices are more likely to meet the training eligibility requirements at night local time. For example, for a language model \citep{gboard}, devices from English speakers in India and America are likely available at different times of day.

When the number of blocks in each cycle is high and the number of
samples per block is low, the setting approaches fully-cyclic SGD,
which is known to be problematic.  But what
happens when the number of blocks is fairly small and each blocks
consists of a large number of samples?  E.g., if there are only two
blocks corresponding to ``day'' and ``night'', with possibly millions
of samples per block.  %
An optimist would hope that this ``slightly cyclic'' case is much easier
than the fully cyclic case, and the performance of SGD degrades gracefully
as the number of blocks increases.

Unfortunately, in Section \ref{sec:lower} we show that even with only
two blocks, the block-cyclic 
sampling can cause an arbitrarily bad slowdown.  One might ask whether alternative
forms of updates, instead of standard stochastic gradient updates, can alleviate this degradation in performance.  We show that such a slowdown
is unavoidable for any iterative method based on semi-cyclic samples.

Instead, the solution we suggest is to embrace the heterogeneity in
the data and resort to a {\em pluralistic} solution, allowing a
potentially different model for each block in the cycle (e.g.~a
``daytime'' model and a ``nighttime'' model).  A naive pluralistic
approach would still suffer a significant slowdown as it would not
integrate information between blocks (Section \ref{sec:plur}).
In Section~\ref{sec:pluravg} we show a remarkably simple and practical pluralistic approach that allows us to obtain exactly the same guarantees as with i.i.d.~non-cyclic sampling, thus entirely alleviating the problem
introduced by such data heterogeneity---as we also demonstrate empirically in Section \ref{sec:experiments}.  In Section \ref{sec:hedge} we go even further and show how we can maintain the same guarantee without any deterioration while also being competitive with separately learned predictors, hedging our bets in case the differences between components are high.

\section{Setup}

We consider a stochastic convex optimization problem
\begin{equation}
  \label{eq:F}
  F(w) = \Ee{z \sim \DD}{f(w,z)}
  \quad\text{where}\quad 
  \DD = \frac{1}{m} \sum_{i=1}^m \DD_i,
\end{equation}
where each component $\DD_i$ represents the data distribution
associated with one of $m$ blocks $i=1,\ldots,m$.  For simplicity, we assume a uniform mixture; our results can be easily extended to non-uniform mixing weights and corresponding non-uniform block lengths.  In a learning setting, $z=(x,y)$ represents a labeled example and the instantaneous objective $f(w,(x,y))=\textit{loss}(h_w(x),y)$ is the loss incurred if using the model $w$.  %

We assume $f(w,z)$ is convex and 1-Lipschitz with respect to $w$,
which lives in some high-, possibly infinite-dimensional, Euclidean or
Hilbert space.  Our goal is to compete with the best possible
predictor of some bounded norm $B$, that is to learn a predictor
$\hat{w}$ such that $F(\hat{w})\leq F(\wopt) + \epsilon$
where
  $F(\wopt) = \inf_{\norm{w}\leq B} F(w)$.
We consider Stochastic Gradient Descent (SGD) iterates on the above objective:
\begin{equation}
  \label{eq:SGD}
  w_{t+1} \leftarrow w_{t} - \eta_{t} \nabla f(w_t,z_t).
\end{equation}
If the samples $z_t$ are chosen independently
from the data distribution $\DD$, then with an appropriately chosen
stepsizes $\eta_t$, SGD attains the mini-max optimal error guarantee \citep[e.g.][]{shalev2009stochastic}:
\begin{equation}
  \label{eq:SGDbound}
  \EE{F(\overline{w})} \leq F(\wopt) + O\left(\sqrt\frac{B^2}{T}\right)
\end{equation}
where $\overline{w} = \frac{1}{T} \sum_{i=1}^T w_i$ and $T$ is the
total number of iterations, and thus also the total number of samples
used.

But here we study {\em Block-Cyclic SGD}, which consists of $K$ {\em
  cycles} (e.g., days) of $mn$ iterations each, for a total of $T=Kmn$
iterations of the update \eqref{eq:SGD}.  In each cycle, the first $n$ samples
are drawn from $\DD_1$, the second $n$ from $\DD_2$ and so forth.
That is, samples are drawn according to
\begin{equation}\label{eq:seqkmn}
  z_{t(k,i,j)}
  \sim\DD_i \quad\textrm{where } t(k,i,j)\!=(k\!-\!1)nm+(i\!-\!1)n+j
\end{equation}
where $k=1..K$ indexes the cycles, $i=1..m$ indexes blocks and
$j=1..n$ indexes iterations within a block and
thus $t(k,i,j)\in\{1,\ldots,T\}$ indexes the overall 
sequence of samples and corresponding updates.  
The samples are therefor no longer identically distributed and the standard SGD analysis no longer valid. We study the effect of such block-cyclic sampling on the SGD updates \eqref{eq:SGD}.

One can think of $\DD$ as the population distribution, with each step
of SGD being based on a fresh sample; or of $\DD$ as an
empirical distribution over a finite training set that is partitioned to $m$
groups.  It is important however not to confuse a
{\em cycle} over the different components with an {\em epoch} over all
samples---the notion of an {\em epoch} or the size of the support of
$F$ (i.e.~the number of training points if $F(w)$ is viewed as an
empirical objective) do not enter our analysis.  Either
view is consistent with our development, though our discussion will
mostly refer to the population view of SGD based on fresh samples.

\section{Lower Bounds for Block-Cyclic Sampling}\label{sec:lower}

How badly can the block-cyclic sampling \eqref{eq:seqkmn} hurt the SGD
updates \eqref{eq:SGD}?  Unfortunately, we will see that the effect
can be quite detrimental, especially when the number of samples $n$ in each
block is very large compared to the overall number of cycles $K$, and
even if the number of components $m$ is small.  This is a typical
regime in Federated Learning if we consider daily cycles, as we would
hope training would take a small number of days, but the number of SGD
iterations per block can be very high. 
For example, recent models for Gboard were trained for \num{1.2e6} sequential steps over 4-5 days ($K = 4$ or $5$, $nm = \num{3e5}$, and e.g. $m=6$ and $n=\num{5e4}$ if the day is divided into six 4-hour blocks) \cite{gboard}.\footnote{The cited work uses the Federated Averaging algorithm with communication every 400 sequential steps, as opposed to averaging gradients every step, so while not exactly matching our setting, it provides a rough estimate for realistic scenarios.}

To see how even $K=2$ can be very bad, 
consider an extreme setting where the
objective $f(w,z)$ depends only on the component of $z$, but is the
same for all samples from each $\DD_i$.  Formally, let $z=\{1,2\}$
with $P(z=i)=1$. That is there are only two possible examples which we
call ``1'', which is the only example in the support of the first component $\DD_1$ and ``2'', which is the only example in the support of $\DD_2$.  The problem then reduces to a deterministic optimization
problem of minimizing $F(w)=\frac{1}{2}f(w,1)+\frac{1}{2}f(w,2)$, but
in each block we see gradients only of one of the two terms.  With a
large number of iterations per block $n$, we might be able to optimize
each term separately very well.  But we also know that in order to
optimize such a composite objective to within $\epsilon$, even if we
could optimize each component separately, requires alternating between
the two components at least $\Omega(B/\epsilon)$ times
\cite{woodworth2016tight}, establishing a lower bound on the number
of required cycles, independent of the number of iterations per cycle.  This extreme deterministic situation establishes the limits of what can be done with a limited number of cycles, and captures the essence of the difficulty with block-cyclic sampling.

In fact, the lower bound we shall see applies not only to the precise
SGD updates \eqref{eq:SGD} but to {\em any} method that uses
block-cyclic samples as in \eqref{eq:seqkmn}, and at each iteration
access $f(\cdot,z_t)$, i.e.~the objective on example $z_t$.  More
precisely, consider any method that at iteration $t$ chooses a ``query
point'' $w_t$ and evaluates $f(w_t,z_t)$ and the gradient (or a
sub-gradient) $\nabla f(w_t,z_t)$, where $w_t$ may be chosen in {\em
  any} way based on all previous function and gradient evaluations
(the SGD iterates \eqref{eq:SGD} is just one examples of such a
method).  The output $\hat{w}$ of the algorithm can then be chosen as
any function of the query points $w_t$, the evaluated function values and the
gradients.

\begin{theorem} \label{thm:LowerLip}
  Consider any (possibly randomized) optimization method of the form
  described in the previous paragraph, i.e.~where access to the
  objective is by evaluating $f(w_t,z_t)$ and $\nabla f(w_t,z_t)$ on
  semi-cyclic samples \eqref{eq:seqkmn} and where $w_t$ is chosen
  based on $\left\{(f(w_s,z_s),\nabla f(w_s,z_s)),s<t\right\}$ and the
  output $\hat{w}$ based on all iterates\footnote{This theorem, as well
    as Theorem \ref{thm:Lowersmooth}, holds even if the method is
    allowed ``prox queries''\removed{, that is to solve an optimization problem
    over $f(w,z_t)$} of the form $\arg\min_w f(w,z_t)+\lambda_t
    \norm{w-w_t}^2$.\removed{where $w_t$ and $\lambda_t$ are chosen by the
    method}}.  For any $B,n,K$ and $m>1$ there exists a 1-Lipschitz
  convex problem over high enough dimension such that $\EE{F(\hat{w})}\geq F(\wopt)+\Omega(B/K)$, where
  the expectation is over $z_t$ and any randomization in the method.
\end{theorem}
\begin{proof}
  Let $P(z=1|i<m/2)=1$ and $P(z=2|i\geq m/2)=1$, with the following functions taken from \citet{woodworth2018graph}, which in turn is based on the constructions in \citet{arjevani2015communication,woodworth2017lower,carmon2017lower,woodworth2017lower}:
 \begin{equation}
      \begin{aligned}
      f(w,1)&=\frac{\eta}{8}\Big(-2a \langle v_1,w \rangle + \phi(\langle v_{4K},w\rangle)
      \\ &\quad  +\sum_{k=1}^{2K-1}\phi(\langle v_{2k}-v_{2k+1},w \rangle)\Big) \\
      f(w,2)&=\frac{\eta}{8}\left(\sum_{k=1}^{2K}\phi(\langle v_{2k-1}-v_{2k},w \rangle)\right)
      \end{aligned}
  \end{equation}
where $v_r$ are orthogonal vectors,
$\eta=4BK,\gamma=2B/(\eta\sqrt{K}),a=1/\sqrt{64K^3}$, and for now
consider $\phi(x)=2 \gamma \abs{x}$.  The main observation is that each vector
$v_{2k+i}$ is only revealed after $w$ includes a component in
direction $v_{2k+i-1}$ (more formally: it is not revealed if
$w\in\spn\{v_1,\ldots,v_{2k+i-2}\}$), and only when $f(w,i)$ is
queried \citep[Lemma 9]{woodworth2018graph}.  That is, each cycle will reveal at most two vectors,
$v_{2k+1}$ for queries on the first half of the blocks, and $v_{2k+2}$
for queries on the second half.  After $K$ cycles, the method would only encounter vectors in the span of the first $2K$ vectors $v_1,\ldots,v_{2K}$.
But for $\hat{w}\in\spn\{v_1,\ldots,v_{2K}\}$, we have $F(\hat{w})\geq F(\wopt)
+ \frac{B}{96K}$ \citep[Lemma 8]{woodworth2018graph}.  These arguments apply if the method does not leave the span of gradients returned so far.  Intuitively, in high enough dimensions, it is futile to investigate new directions aimlessly.  More formally, to ensure that trying out new directions over $T=Kmn$ queries wouldn't help, following appendix C of \citep{woodworth2018graph}, we can choose $v_r$ randomly in $R^{\tilde{O}(K^5n^2m^2)}$ and use a piecewise quadratic $\phi(x)$ that is 0 for $\abs{z}\leq a/2$ and is equal to $\phi(x)=2\gamma \abs{x}-\gamma^2 - a^2/2$
for $\abs{x}\geq\gamma$. 
\end{proof}

Theorem \ref{thm:LowerLip} establishes that once we have a
block-cyclic structure with multiple blocks per cycle, we cannot
ensure excess error better then:
\begin{equation}
    \label{eq:lower1}
    F(\hat{w})-F(\wopt)=\Omega\left(\frac{B}{K}\right)=\Omega\left(\sqrt\frac{B^2}{T}\sqrt{\frac{mn}{K}}\right).
\end{equation}
Compared to using i.i.d.~sampling as in Eq.~\eqref{eq:SGDbound}, this is worse by a factor of $\sqrt{mn/K}$, 
which is very large when the number of samples per block
$n$ is much larger then the number of cycles $K$, as we would expect in many applications. \removed{Said differently, reaching excess error $\epsilon$
might take $K = B/\epsilon = \sqrt{nm K_{\textrm{iid}}}$ cycles, where
$K_{\textrm{iid}}$ is the number of cycles required with
i.i.d.~sampling and $nm$ iterations per cycle.}  

For smooth objectives (i.e., with Lipschitz gradients) the situation is quantitatively a bit better,
but we again can be arbitrarily worse then i.i.d.~sampling as the number $n$ of iterations per
block increases:
\begin{theorem}\label{thm:Lowersmooth}
  Under the same setup as in Theorem \ref{thm:LowerLip}, for any
  $B,n,K$ and $m>1$ there exists a 1-Lipschitz convex problem where
  the gradient $\nabla_w f(w,z)$ is also 1-Lipschitz, such that
  $\EE{F(\hat{w})}\geq F(\wopt)+\Omega(B^2/K^2)$.
  \end{theorem}
\begin{proof}
  Use the same construction as in Theorem \ref{thm:LowerLip}, but with $\eta=B^2$ and 

\begin{minipage}[c]{0.35\textwidth}
\vspace{-5mm}
$$ \phi(z) = 
\begin{cases} 
0 & \abs{z} \leq a/2 \\
2(\abs{z} - a/2)^2 & a/2 < \abs{z} \leq a \\
z^2 - a^2/2 & a < \abs{z} \leq \gamma \\
2\gamma\!\abs{z} -\! \gamma^2\! - a^2/2 & \abs{z} > \gamma
\end{cases}  $$
\end{minipage}
\begin{minipage}[c]{0.12\textwidth}
\includegraphics[height=20mm,width=24mm]{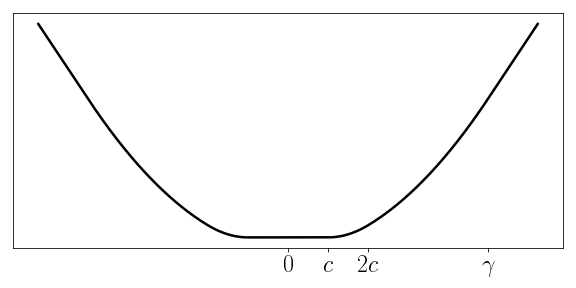}
\end{minipage} %
The objective is smooth \citep[Lemma 7]{woodworth2018graph}, and the same arguments as in the proof of Theorem \ref{thm:LowerLip} hold except that now $\hat{w}$ spanned by $v_1,\ldots,v_{2K}$ has $F(\hat{w})\geq F(\wopt)
+ \frac{B^2}{256K}$ \citep[Lemma 8]{woodworth2018graph}.
\end{proof}
That is, the best we can ensure with block-cyclic sampling, even if the objective is smooth, is an excess error of:
\begin{equation}\label{eq:lowerboundsmooth}
\epsilon = \Omega\left( \frac{B^2}{K^2} \right) = \Omega\left(
\sqrt\frac{B^2}{T} \sqrt\frac{B^2 mn}{K^3}
\right),
\end{equation}
i.e., worse by a $\sqrt{\frac{B^2 mn}{K^3}}$ factor compared to i.i.d.~sampling.

\section{A Pluralistic Approach}\label{sec:plur}

The slowdown discussed in the previous section is due to the
difficulty of finding a {\em consensus} solution $w$ that is good for
all components $\DD_i$.  But thinking of the problem as a learning
problem, this should be avoidable.  Why should we insist on a single
consensus predictor?  In a learning setting, our goal is to be able to
perform well on future samples $z\sim \DD$.  If the data arrives
cyclically in blocks, we should be able to associate each sample $z$
with the mixture component $\DD_i$ it was generated from.  Formally,
one can consider a joint distribution $(i,z)\sim\tDD$ over samples and
components, where $i$ is uniform over $1..m$ and $z\,|\,i\sim\DD_i$.  In
settings where the data follows a block-cyclic structure, it is
reasonable to consider the component index $i$ as observed and
leverage this at prediction (test) time.

For example, in Federated Learning, this could be done by
associating each client with the hours it is  available for training. If it
is available during multiple blocks, we can associate with it a
distribution over $i$ reflecting its availability.

Instead of learning a single consensus model for all components, we
can therefore take a pluralistic approach and learn a separate model
$w^i$ for each one of the components $i=1..m$ with the goal of
minimizing:
\begin{equation}
  \label{eq:plurGoal}
  \tF(\allw)=\Ee{(i,z)\sim\tDD}{f(w^i,z)} = \frac{1}{m}\sum_{i=1}^m F_i(w^i) ,
\end{equation}
where $\allw=(w^1,\ldots,w^m)$ and $F_i(w^i)=\Ee{z\sim\DD_i}{f(w^i,z)}$.

How can this be done?  A simplistic approach would be to learn a
predictor $w^i$ for each component $D_i$ separately, based only on
samples from block $i$.  This could be done by maintaining $m$
separate SGD chains, and updating chain $i$ only using samples from
block $i$:
\begin{multline}
  w^{i}_{t(k,i,j+1)} \leftarrow w^{i}_{t(k,i,j)}-\eta_{t(k,i,j)} \nabla
  f(w^i_{t(k,i,j)},z_{t(k,i,j)}) \\*
\textrm{where } w^{i}_{t(k+1,i,1)}=w^{i}_{t(k,i,n+1)} . \label{eq:separatew}
\end{multline}
We will therefore only have $Kn=T/m$ samples per model, but for each
model, the samples are now i.i.d., and we can learn
$\hat{w}^i$ such that:
\begin{equation}\label{eq:eachFi}
F_i(\hat{w}^i) \leq F_i(\wopt^i) + O\left(\sqrt\frac{B^2}{T/m}\right)
\end{equation}
where $\wopt^i = \arg\min_{\norm{w}\leq B} F_i(w)$, and so:
\begin{align}
  \tF(\hat{\allw}) 
  &= \frac{1}{m}\sum_{i=1}^m F_i(\hat{w}^i) \notag
  \\
  &\leq \frac{1}{m}\sum_{i=1}^m F_i(\wopt^i) + \label{eq:sepFi}
  O\left(\sqrt\frac{B^2}{T/m}\right) \\
  &\leq F(\wopt) +   O\left(\sqrt\frac{B^2}{T/m}\right). \label{eq:sepFstar}
\end{align}
That is, ensuring excess error $\epsilon$ requires
$T=O(m B^2/\epsilon^2)$ iterations, which represents
a slow-down by a factor of $m$ relative to overall i.i.d.~sampling.

The pluralistic approach can have other gains, owing to its multi-task
nature, if the $m$ sub-problems $\min F_i(w)$ are very different from
each other.  An extreme case is when the different subproblems are in
direct conflict, in a sense ``competing'' with each other, and the
optimal $\wopt^i$ for the different sub-populations are not compatible
with each other.  For example, this could happen if a very informative feature has
different effects in different subgroups.  In this case we might have
$F_i(\wopt^i) \ll F_i(\wopt)$, and there is a strong benefit to
the pluralistic approach regardless of the semi-cyclic sampling.

The lower bounds of Section \ref{sec:lower} involve a subtly different
conflicting situation, where there is a decent consensus predictor, but
finding it requires, in a sense, ``negotiations'' between the
different blocks, and this requires going back-and-forth between
them many times, as is possible with i.i.d.~non-cyclic data, but not
in a block-cyclic situation.  Learning separate predictors would
bypass this required conflict resolution.

An intermediate situation is when the sub-problems are ``orthogonal''
to each other, e.g.~different features are used in different problems.
In this case we have that $\wopt^i$ are orthogonal to each other, and we
might have that $\wopt = \sum_i \wopt^i$ and $F_i(\wopt^i)=F_i(\wopt)$.
However, in this case, we would have that on average (over $i$),
$\norm{\wopt^i}=\frac{\norm{\wopt}}{\sqrt{m}}$, and so with the
pluralistic approach we can learn relative to a norm-bound
$B'=B/\sqrt{m}$ that is $\sqrt{m}$ times smaller than would be
required when using a single consensus model.  This precisely negates
the slow-down in terms of $m$ of the pluralistic learn-separate-predictors
approach \eqref{eq:separatew}, and we recover the same performance as when learning a
consensus predictor based on i.i.d.~non-cyclic data.

The regime where we {\em would} see a significant slow-down is when the distributions $\DD_i$
are fairly similar.  By separating the problem into $m$
distinct problems, and not sharing information between
components, we are effectively cutting the amount of data, and updates, by a factor of $m$, causing the slow-down.  But if the distributions are indeed similar, then at least intuitively, the semi-cyclisity shouldn't be too much of a problem, and a simple single-SGD approach might be OK.  At an extreme, if all components are identical ($\DD_i=\DD$), the block-cyclic sampling is actually i.i.d.~sampling and we do not have a problem in the first place.

We see, then, that in extreme situations, semi-cyclicity is not a real
problem: if the components are extremely ``competing'', we would be
better off with separate predictors, while
if they are identical we can just use a single SGD chain and lose
nothing.  But how do we know which situation we are in?  Furthermore,
what we actually expect is a combination of the above 
scenarios, with some aspects being in direct competition between the
components, some being orthogonal while others being aligned and
similar.  Is there a simple approach
that would always allow us to compete with training a single model
based on i.i.d.~data?  That is, a procedure for learning
$\allw=(w^1,\ldots,w^m)$ such that:
\begin{equation}
  \label{eq:plur_goal}
  \tF(\allw) \leq F(\wopt) + O\left(\!\sqrt{\frac{B^2}{T}}\,\right).
\end{equation}

We would like to make a distinction between our objective here and that of multi-task learning.  In multi-task learning \citep[and many others]{caruana1997multitask}, one considers several different but possibly related tasks (e.g.~the tasks specified by our component distributions $\DD_i$) and the goal is to learn different ``specialized'' predictors for each task so as to {\em improve} over the consensus error $F(\wopt)$ while leveraging the relatedness between them so as to reduce sample complexity.  In order to do so, it is necessary to target a particular way in which tasks are related \citep{baxter2000model,ben2003exploiting}.  For example, one can consider shared sparsity \citep[e.g.][]{turlach2005simultaneous}, shared linear features or low-rank structure \citep{ando2005framework}, shared deep features \citep[e.g.][]{caruana1997multitask}, shared kernel or low nuclear-norm structure \citep[e.g.][]{argyriou2008convex,amit2007uncovering},  low-norm additive structure \citep[e.g.][]{evgeniou2004regularized}, or graph-based relatedness \citep[e.g.][]{evgeniou2005learning,maurer2006bounds}.  The success of multi-task learning then rests on whether the choosen relatedness assumptions hold, and the sample complexity depends on the specificity of this inductive bias.  But we do {\em not} want to make any assumptions about relatedness.  We would like a pluralistic method that {\em always} achieve the guarantee \eqref{eq:plur_goal}, without any additional or even low-order terms that depend on the relatedness.  We hope we can achieve this since in \eqref{eq:plur_goal}, we are trying to compete with the {\em fixed} solution $F(\wopt)$, and are resorting to pluralism only in order to overcome data heterogeneity, not in order to leverage it.

\section{Pluralistic Averaging}\label{sec:pluravg}
We give a surprisingly simple solution to the above problem.  
It {\em is} possible to compete with $F(\wopt)$
in a semi-cyclic setting, without any additional performance
deterioration (at least on average) and with no assumptions about the
specific relatedness structure between different components.  In fact,
this can be done by running a single semi-cyclic SGD chain
\eqref{eq:SGD}, which previously we discussed was problematic.  The
only modification is that instead of averaging all iterates to obtain
a single final predictor (or using a single iterate), we create $m$
different pluralistic models by averaging, for each component $i$,
only the iterates corresponding to that block:
\begin{equation}
  \label{eq:wtilde}
  \tilde{w}^i 
  = 
  \frac{1}{Kn} \sum_{k=1}^K\sum_{j=1}^n w_{t(k,i,j)}
  .
\end{equation}

\begin{theorem}\label{thm:pluravg}
  Consider semi-cyclic samples as in \eqref{eq:seqkmn}.  The pluralistic averaged solution $\tilde{\allw}$ given in
  \eqref{eq:wtilde} in terms of the iterates of \eqref{eq:SGD} with step size $\eta_t=B/\sqrt{2T}$ and starting at $w_1=0$,
  satisfies
  \begin{equation}
    \label{eq:mainE}
    \EE{\tF(\tilde{\allw})} \leq F(\wopt) + O\left(\!\sqrt{\frac{B^2}{T}}\,\right),
  \end{equation}
  where the expectation is w.r.t.~the samples.  \removed{In addition, for a
  bounded objective $0\leq f(w,z)\leq a$, and any $\delta>0$, with
  probability at least $1-\delta$ over the samples:
  \begin{equation}
    \label{eq:mainP}
    \tF(\tilde{\allw}) \leq F(\wopt) +
    O\left(\sqrt{\frac{B^2+a \log 1/\delta}{T}}\right).
  \end{equation}}
\end{theorem}

The main insight is that the SGD guarantee can be obtained through
an online-to-batch conversion, and that the online guarantee itself
does {\em not} require any distributional assumption and so is valid also
for semi-cyclic data.  Going from Online Gradient Descent to Stochastic Gradient Descent, the i.i.d.~sampling {\em is} vital for the online-to-batch conversion. But by averaging only over iterates corresponding to samples from the same component, we are in a sense doing an online-to-batch conversion for each component separately, and thus over i.i.d.~samples.
  
\begin{proof}
Viewing the updates \eqref{eq:SGD} as implementing online gradient descent \cite{zinkevich2003online}\footnote{The development in \citet{zinkevich2003online} is for updates that also involve a projection: $w_{t+1} \leftarrow \Pi_B(w_t - \eta_t \nabla f(w_t,z_t))$ where $\Pi_B(w)=w/\!\max(\norm{w}\!/B,1)$.  See, e.g., \citet{shalev2012online}, for a development of SGD without the projection as in \eqref{eq:SGD}, and proof of the regret guarantee \eqref{eq:OGDreg} for these updates.  Although we present the analysis without a projection, our entire development is also valid with a projection as in \citet{zinkevich2003online}, which results in all the same guarantees}, we have the following online regret guarantee for {\em any} sequence $z_t$ (and in particular any sequence obtained from any sort of sampling) and any $w$ with $\norm{w}\leq B$:
\begin{equation}\label{eq:OGDreg}
\sum_{t=1}^T  f(w_t,z_t) \leq \sum_{t=1}^T f(w,z_t) + \sqrt{2B^2 T}.
\end{equation}
choosing $w=\wopt$ on the right hand side, dividing by $T$ and rearranging the summation we have:
\begin{multline}
    \label{eq:OGDbyT}
    \frac{1}{m} \sum_{i=1}^m \frac{1}{Kn} \sum_{k,j} f(w_{t(k,i,j)},z_{t(k,i,j)})\\*
    \leq \frac{1}{m} \sum_{i=1}^m \frac{1}{Kn} \sum_{k,j} f(\wopt,z_{t(k,i,j)}) + \sqrt{\frac{2B^2}{T}}.
\end{multline}
The above is valid for any sequence $z_t$. Taking $z_t$ to be random, regardless of their distribution, we can take expectations on both sides. For the semi-cyclic sampling \eqref{eq:seqkmn}, and since $w_t$ is independent of $z_t$ we have:
  \begin{multline}\label{eq:OGDbyTex}
    \EE{\frac{1}{m} \sum_{i=1}^m \frac{1}{Kn} \sum_{k,j} \Ee{z\sim\DD_i}{f(w_{t(k,i,j)},z)}} \\* 
    \leq \frac{1}{m} \sum_{i=1}^m \frac{1}{Kn} \sum_{k,j} 
    \Ee{z\sim\DD_i}{f(\wopt,z)} + \sqrt{\frac{2B^2}{T}},
\end{multline}
where the outer expectation on the left-hand-side is w.r.t.~$w_t$.  On the right hand side we have that $\Ee{z\sim\DD_i}{f(\wopt,z)}=F_i(\wopt)$.  On the left hand side, we can use the convexity of $f$ and apply Jensen's inequality:
\begin{multline}
    \label{eq:Jensen}
    \frac{1}{Kn} \sum_{k,j} \Ee{z\sim\DD_i}{f(w_{t(k,i,j)},z)} \\
    \geq \Ee{z\sim\DD_i}{\,f\Big(\frac{1}{Kn} \sum_{k,j} w_{t(k,i,j)},z\Big)\!} 
    = F_i(\tilde{w}_i)
    .
\end{multline}
Substituting \eqref{eq:Jensen} back into \eqref{eq:OGDbyTex} we have:
\begin{equation}
    \label{eq:FiFiBound}
    \EE{\frac{1}{m}\sum_{i=1}^mF_i(\tilde{w}_i)} \leq \frac{1}{m}\sum_{i=1}^mF_i(\wopt) + \sqrt{\frac{2B^2}{T}}.
\end{equation}
recalling the definition \eqref{eq:plurGoal} of $\tF(\tilde{\allw})$ and that $F(w)=\frac{1}{m}\sum_i F_i(w)$, we obtain the expectation bound \eqref{eq:mainE}. 
\end{proof}

An important, and perhaps surprising, feature of the guarantee of Theorem \ref{thm:pluravg} is that it does {\em not} depend on the number of blocks $m$, and does not deteriorate as $m$ increases.  That is, in terms of learning, we could partition our cycle to as many blocks as we would like, ensuring homogeneity inside each block and without any accuracy cost.  E.g.~we could learn a separate model for each hour, or minute, or second of the day.  The cost here is only a memory and engineering cost of storing, distributing and handling the plenitude of models, not an accuracy nor direct computational cost.  This is in sharp contrast to separate training for each component, as well as to most multi-task approaches.

\section{Pluralistic Hedging} \label{sec:hedge}

\newcommand{\whedge}{u}
\newcommand{\allwhedge}{\mathbf{u}}
\newcommand{\lrhedge}{\nu}

The pluralistic averaging approach shows that we can always compete
with the best single possible solution $\wopt$, even when faced with
semi-cyclic data.  But as discussed in Section \ref{sec:plur}, perhaps
in some extreme cases learning separate predictors $w^i$, each based
only on data from the $i$th component as in \eqref{eq:separatew}, might
be better.  That is, depending on the similarity and conflict between
components, the guarantees \eqref{eq:eachFi} might be better than that
of Theorem \ref{thm:pluravg}, at least for some of the predictors.  
But if we do not know in advance which
regime we are at, nor for which components we are better off with a separately-learned model (because they are very different from the others) and which are better off leveraging also the other components, can we still ensure the $m$ individual guarantees \eqref{eq:eachFi} and the guarantee of Theorem \ref{thm:pluravg} simultaneously?

Here we give a method for doing so, based on running both the single SGD chain \eqref{eq:SGD} and the separate SGD chains \eqref{eq:separatew} and carefully combining them using learned weights for each chain.
Let $q>0$ be the weight for the full SGD chain $w_t$~\eqref{eq:SGD}, which will be kept fixed throughout, and let $q^i_t>0$ be the weights assigned to the block-specific SGD chains $w_t^i$~\eqref{eq:separatew} on step $t$.
We learn the weights using a multiplicative update rule. 
At step $t=t(k,i,j)$, in which block $i$ is active, this update takes the form:
\begin{align*}
	q^{i}_{t+1} 
	&\gets 
	q^{i}_t \cdot \big( 1+\lrhedge\big(f(w_t,z_t)-f(w_t^{i},z_t)\big) \big) ;
	\\
	\forall ~ i' \neq i ,
	\qquad 
	q^{i'}_{t+1}
	&\gets
	q^{i'}_t
	~,
\end{align*}
where $\lrhedge>0$ is a learning rate (separate from those of the SGD chains).
Then, we let
$
	p_t 
	= 
	q_t^{i} \big/ (q_t^{i} + q)
	\in [0,1]
$
and choose between the full and block-specific SGD chains by:
\begin{align*}
\whedge_t \gets \begin{cases}
	w^{i}_t & \text{with probability $p_t$;}\\
	w_t & \text{otherwise.}
\end{cases}
\end{align*}
Finally, we obtain the final predictors via pluralistic averaging; namely, for each component $i$, we average only the iterate within the corresponding blocks:
\begin{align} \label{eq:wtilde-hedge}
    \tilde{\whedge}^i
    = 
    \frac{1}{Kn} \sum_{k=1}^K \sum_{j=1}^n \whedge_{t(k,i,j)}
    ~.
\end{align}
For the averaged solution $\tilde{\allwhedge} = (\tilde{\whedge}^1,\ldots,\tilde{\whedge}^m)$, we have the following.

\begin{theorem} \label{thm:hedging}
Set $\nu = \frac{1}{2B} \sqrt{(m/T)\log(BT/m)}$, $q = 1-\eta$, and initialize $q^i_1 = \eta$ for each $i$.
Then, for all $i$ we have
\begin{align*}
	\EE{F_i(\tilde{\whedge}^i)} 
	&\leq 
	F_i(\wopt^i) + 4\sqrt{\frac{B^2 \log(BT/m)}{T/m}}
	,
\intertext{and further, provided that $m \leq B^2Kn$,}
    \EE{\tF(\tilde{\allwhedge})}
    &\leq 
    F(\wopt) + 4\sqrt{\frac{B^2}{T}}
    .
\end{align*}
\end{theorem}

The requirement that $m\leq B^2Kn$ is extremely mild, as we would generally expect a large number of iterations per block $n$, and only a mild number of blocks $m$, i.e.~$n\ll m$.

The proofs in this section use the following notation.
For all $i \in [m]$, let $S_i \subseteq [T]$ be the set of time steps where we got a sample from distribution $\DD_i$ (so that $\abs{S_i} = T/m$).
We first prove the following.

\begin{lemma} \label{lem:regret}
For each $i \in [m]$ we have
\begin{align*}
    \EE{ \sum_{t \in S_i} f(\whedge_t, z_t) - \sum_{t \in S_i} f(w_t, z_t) }
    &\leq
    2
    ~,
    \intertext{and}
    \EE{ \sum_{t \in S_i} f(\whedge_t, z_t) - \sum_{t \in S_i} f(w^i_t, z_t) }
    &\leq
    2B \sqrt{\frac{T}{m} \log\frac{BT}{m}}
    ~.
\end{align*}
Here, the expectation is taken w.r.t.~the $z_t$ as well as the internal randomization of the algorithm.
\end{lemma}

The proof (in \cref{sec:hedge-proofs}) is based on classic analysis of the {\sc Prod} algorithm (\citealp{Cesa-Bianchi2007}; see also~\citealp{Even-Dar2008,sani2014exploiting}).
We can now prove the main theorem of this section.

\begin{proof}[Proof of \cref{thm:hedging}]
The proof follows from a combination of \cref{lem:regret} with regret bounds for the SGD chains \eqref{eq:SGD} and \eqref{eq:separatew}, viewed as trajectories of the Online Gradient Descent algorithm.
Standard regret bounds for the latter (\citealp{zinkevich2003online}; see also \citealp{shalev2012online,hazan2016introduction}) imply that for any sequence $z_1,\ldots,z_T$ and for any $\norm{\wopt} \leq B$, it holds that
\begin{align}
	\sum_{t=1}^T f(w_t, z_t) - \sum_{t=1}^T f(\wopt, z_t) 
	&\leq
	2B \sqrt{T} 
	~, \label{eq:ogd-all}
	\intertext{and, for each $i \in [m]$ and for any $\wopt^i$ such that $\norm{\wopt^i} \leq B$,}
	\sum_{t \in S_i} f(w^i_t, z_t) - \sum_{t=1}^T f(\wopt^i, z_t) 
	&\leq
	2B \sqrt{\frac{T}{m}}
	~. \label{eq:ogd-separate}
\end{align}
Now, fix $i \in [m]$; \cref{lem:regret} together with \cref{eq:ogd-separate} imply
\begin{align*}
    \frac{1}{T/m} \sum_{t \in S_i} \EE{ f(\whedge_t, z_t) - f(\wopt^i, z_t) }
    &\leq
    4\sqrt{\frac{B^2 \log(BT/m)}{T/m}}
    ~.
\end{align*}
Using the facts that $\Ee{t}{f(\whedge_t,z_t)} = F_i(\whedge_t)$ and $\Ee{t}{f(\wopt^i,z_t)} = F_i(\wopt^i)$ for any $t \in S_i$, we have
\begin{align*}
    \frac{1}{T/m} \sum_{t \in S_i} \big( \EE{ F_i(\whedge_t) } - F_i(\wopt^i) \big)
    \leq
    4\sqrt{\frac{B^2 \log(BT/m)}{T/m}}
    ~.
\end{align*}
Appealing to the convexity of $F_i$ and applying Jensen's inequality on the left-hand side to lower bound $\frac{1}{T/m} \sum_{t \in S_i} F_i(\whedge_t) \geq F_i(\tilde{\whedge}^i)$, we obtain
\begin{align*}
	\EE{ F_i(\tilde{\whedge}^i) } - F_i(\wopt^i)
	\leq
    4\sqrt{\frac{B^2 \log(BT/m)}{T/m}}
	~,
\end{align*}
which implies the first guarantee of the theorem.
For the second claim, the first bound of \cref{lem:regret} implies
\begin{align*}
    \sum_{i=1}^m \EE{ \sum_{t \in S_i} f(\whedge_t, z_t) } - \EE{ \sum_{t=1}^T f(w_t, z_t) }
    \leq
    2m
    .
\end{align*}
Summing this with \cref{eq:ogd-all} and dividing through by $T$ gives
\begin{align*}
    \frac{1}{m} \sum_{i=1}^m \frac{1}{T/m} \sum_{t \in S_i} \EE{ f(\whedge_t, z_t) - f(\wopt, z_t) }
    \leq
    \frac{2m}{T} + 2\sqrt{\frac{B^2}{T}}
    .
\end{align*}
Next, as before, substitute the conditional expectations and use Jensen's to lower bound the left-hand side; this yields
\begin{align*}
    \frac{1}{m} \sum_{i=1}^m \big( \EE{F_i(\tilde{\whedge}^i)} - F_i(\wopt) \big)
    \leq
    \frac{2m}{T} + 2\sqrt{\frac{B^2}{T}}
    .
\end{align*}
Recalling now the definitions $\tF(\tilde{\allwhedge}) = \frac{1}{m} \sum_{i=1}^m F_i(\tilde{\whedge}^i)$ and $F(\wopt) = \frac{1}{m} \sum_{i=1}^m F_i(\wopt)$, we have shown that
\begin{align*}
	\EE{ \tF(\tilde{\allwhedge}) } - F(\wopt)
	\leq
	\frac{2m}{T} + 2\sqrt{\frac{B^2}{T}}
	~.
\end{align*}
Noting $m \leq B\sqrt{T} = B\sqrt{Kmn}$ implies $m \leq B^2Kn$ conludes the proof.
\end{proof}

\section{Experiments}\label{sec:experiments}
To illustrate the challenges of optimizing on block-cyclic data, we train and evaluate a small logistic regression model on the Sentiment140 Twitter dataset \cite{go09}, a binary classification task over \num{1.6e6} examples. We split the data into training (90\%) and test (10\%) sets, partition it into $m=6$ components based on the timestamps (but not dates) of the Tweets: 12am - 4am, 4am - 8am, etc, then divide each component across $K=10$ cycles (days). For more details, see Appendix~\ref{sec:exp_details}. For simplicity, we keep the model architecture (linear bag of words classifier over top 1024 words) and minibatch size (128) fixed; we used a learning rate of $\eta = 0.464$ (determined through log-space grid search) except for the per-component SGD approach %
\eqref{eq:separatew} where $\eta = 1.0$ was optimal due to fewer iterations.

To illustrate the differences between the proposed methods, as well as to capture the diurnal variation expected in practical Federated Learning settings, we vary the label balance as a function of the time of day, ensuring the $\DD_i$ are somewhat different. In particular, we randomly drop negative posts from the midnight group so the overall rate is 2/3 positive sentiment, we randomly drop positive posts from the noon group so the overall rate is 1/3 positive sentiment, and we linearly interpolate for the other 4 components. For discussion, and a figure giving positive and negative label counts over components, see Appendix~\ref{sec:exp_labels}. We write $\hat{F}_i$ for the empirical accuracy on component $i$ of the test set, with $\hat{F}$ measuring the overall test accuracy.

We consider the following approaches:
\vspace{-0.15in}

\begin{enumerate} \itemsep -2pt
    \item A non-pluralistic \textbf{block-cyclic consensus} model, subject to the lower bounds of Section~\ref{sec:lower}. The accuracy of this procedure for day $k$ is given by the expected test-set accuracy of randomly picking $i \in \{1, \dots, m=6\}$, and evaluating $F(w_{t(k,i,n)})$, i.e. evaluating the model after completing a random block $i$ on each day $k$.
    \item The \textbf{per-component SGD} approach \eqref{eq:separatew}, where we run $m=6$ SGD chains, with each model training on only one block per day. For each day $k$, we evaluate $\frac{1}{m}\sum_i \hat{F}_i(w^i_{kn})$, where $w^i_{kn}$ is the model for component $i$ trained on the first $k$ blocks for component $i$.
    \item The \textbf{pluralistic single SGD chain} approach of Section~\ref{sec:pluravg}. The SGD chain is the same as for the consensus model, but evaluation is pluralistic as above: for each day $k$ we evaluate $\frac{1}{m}\sum_i \hat{F}_i(w_{t(k,i,n)})$, i.e.~for each test component we use the model at the end of the corresponding training block.
    \item An (impractical in real settings) \textbf{idealized i.i.d. model}, trained on the complete shuffled training set. For purposes of evaluation, we treat the data as if it had cyclic block structure and evaluate identically to the block-cyclic consensus model.
\end{enumerate}
See Appendix~\ref{sec:exp_eval} for more details on evaluation. Note that we do not use a model average, but rather take the most recent relevant iterate, matching how such models are typically deployed in practice.

\begin{figure}[t]
\centering
\includegraphics[width=8cm]{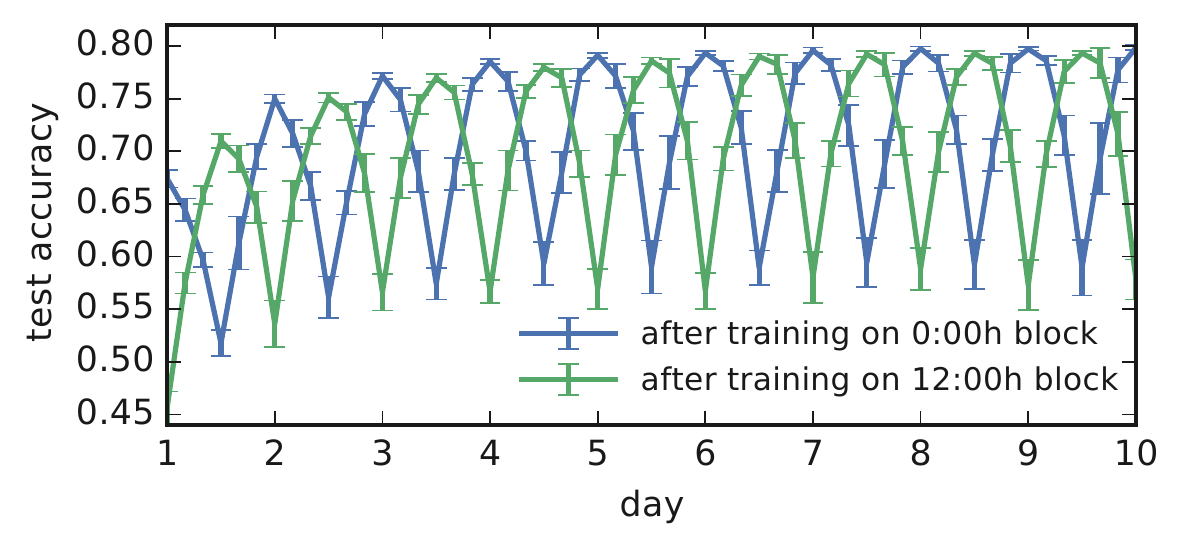}
\vspace{-0.4cm}
\caption{Accuracy of a model trained on block-cyclic data tested on two blocks of data - posts from the midnight component (12am-4am), and posts from the noon component (12pm-4pm). Mean and standard deviations are computed from ten training repetitions.}
\label{fig:experiment_results_cycles}
\end{figure}

Figure \ref{fig:experiment_results_cycles} illustrates how the block-cyclic structure of data can hurt accuracy of a consensus model. We plot how the accuracy on two test data components --- the midnight, and the noon component --- changes as a function of time of day when evaluated on the iterates of the SGD chain \eqref{eq:SGD}.  Training is quick to leverage the label bias and gravitate towards a model too specific for the block being processed, instead of finding a better predictive solution based on other features that would help predictions across all groups.

In Figure \ref{fig:experiment_results} we compare results from the four different training and test methods introduced above. 
First, pluralistic models with separate SGD chains take longer to converge because they are trained on less data. Depending on how long training proceeds, the size of the data set, and the number of components, these models may or may not surpass the idealized i.i.d. SGD and pluralistic single SGD chain models in accuracy.
Second, the pluralistic models from a single SGD chain consistently perform better than a single consensus model from the same chain.
Third, the performance of pluralistic models is en par with or better than the idealized i.i.d. SGD model, reflecting the fact that these models better match the target data distribution than a single model (i.i.d. or block cyclic consensus) can.

\section{Summary}

We considered training in the presence of block-cyclic data, showed that ignoring this source of data heterogeneity can be detrimental, but that a remarkably simple pluralistic approach can entirely resolve the problem and ensure, even in the worst case, the same guarantee as with homogeneous i.i.d.~data, and without any degradation based on the number of blocks.  When the component distributions are actually different, pluralism can outperform the ``ideal'' i.i.d.~baseline, as our experiments illustrate. An important lesson we want to stress is that pluralism can be a critical tool for dealing with heterogeneous data, by embracing this heterogeneity instead of wastefully insisting on a consensus solution. 

Dealing with heterogeneous data, users or clients can be difficult in many machine learning settings, and especially in Federated Learning where learning happens on varied devices with different characteristics, each handling its own distinct data.  Other heterogeneous aspects we can expect are variabilities in processing time and latency, amount of data per client, and client availability and reliability, which can all be correlated in different ways with different data components.  All these pose significant challenges when training.  We expect pluralistic solutions might be relevant in many of these situations.

\begin{figure}[t]
\centering
\vspace{-0.25cm}
\includegraphics[width=8cm]{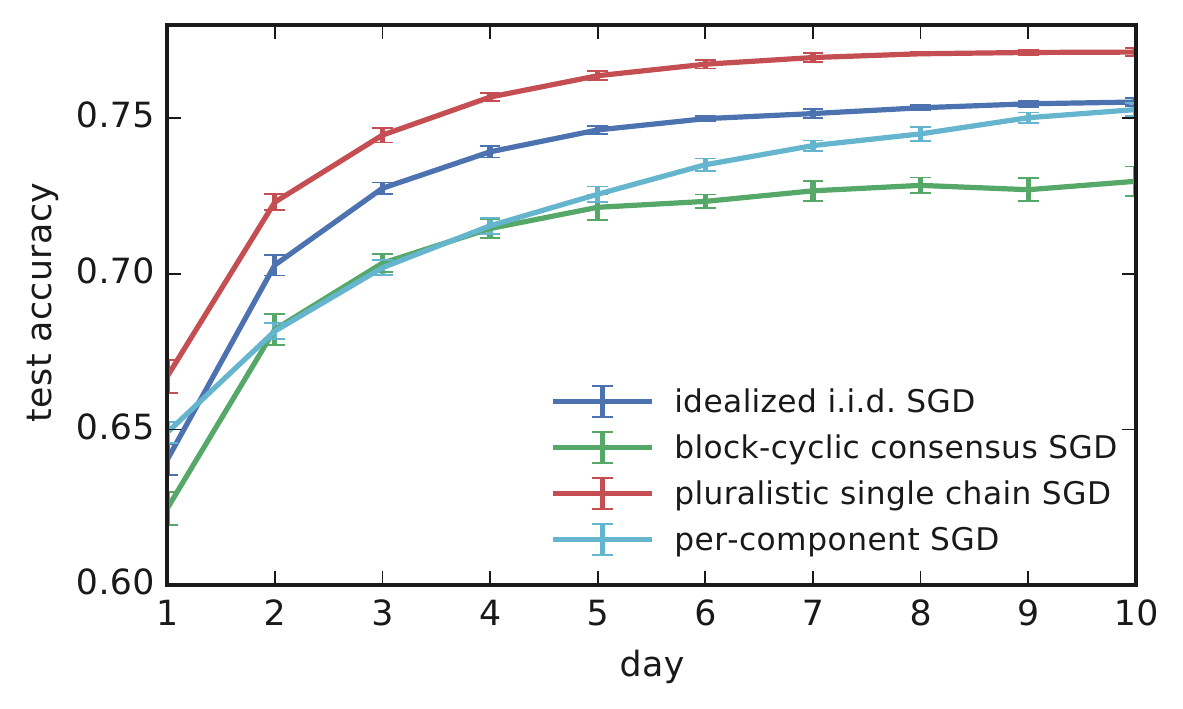}
\vspace{-0.4cm}
\caption{Comparison of various training and testing methodologies for a simple model on a sentiment classification task, using i.i.d. and block-cyclic data. Means and standard deviations shown are computed from ten training repetitions.}
\vspace{-0.5cm}
\label{fig:experiment_results}
\end{figure}

In this paper we considered only optimizing {\em convex} objectives using sequential SGD.  In many cases we are faced with non-convex objectives, as in when training deep neural networks.  As with many other training techniques, we expect this study of convex semi-cyclic training to be indicative also of non-convex scenarios (except that a random predictor would need to be used instead of an averaged predictor in the online-to-batch conversion) and also serve as a basis for further analysis of the non-convex case.  We are also interesting in analyzing the effect of block-cyclic data on methods beyond sequential SGD, and most prominently parallel SGD (aka Federated Averaging).  Unfortunately, there is currently no satisfying and tight analysis of parallel SGD even for i.i.d.~data, making a detailed analysis of semi-cyclicity for this method beyond our current reach.  Nevertheless, we 
again hope our analysis will be both indicative and serve as a basis for future exploration of parallel SGD and other distributed optimization approaches.

\section*{Acknowledgements}

NS was supported by NSF-BSF award 1718970 and a Google Faculty Research Award.

\setlength{\bibsep}{0.75ex plus 0.3ex}
\bibliography{pluralistic}

\begin{thebibliography}{28}
\providecommand{\natexlab}[1]{#1}
\providecommand{\url}[1]{\texttt{#1}}
\expandafter\ifx\csname urlstyle\endcsname\relax
  \providecommand{\doi}[1]{doi: #1}\else
  \providecommand{\doi}{doi: \begingroup \urlstyle{rm}\Url}\fi

\bibitem[Amit et~al.(2007)Amit, Fink, Srebro, and Ullman]{amit2007uncovering}
Amit, Y., Fink, M., Srebro, N., and Ullman, S.
\newblock Uncovering shared structures in multiclass classification.
\newblock In \emph{Proceedings of the 24th international conference on Machine
  learning}, pp.\  17--24. ACM, 2007.

\bibitem[Ando \& Zhang(2005)Ando and Zhang]{ando2005framework}
Ando, R.~K. and Zhang, T.
\newblock A framework for learning predictive structures from multiple tasks
  and unlabeled data.
\newblock \emph{Journal of Machine Learning Research}, 6\penalty0
  (Nov):\penalty0 1817--1853, 2005.

\bibitem[Argyriou et~al.(2008)Argyriou, Evgeniou, and
  Pontil]{argyriou2008convex}
Argyriou, A., Evgeniou, T., and Pontil, M.
\newblock Convex multi-task feature learning.
\newblock \emph{Machine Learning}, 73\penalty0 (3):\penalty0 243--272, 2008.

\bibitem[Arjevani \& Shamir(2015)Arjevani and
  Shamir]{arjevani2015communication}
Arjevani, Y. and Shamir, O.
\newblock Communication complexity of distributed convex learning and
  optimization.
\newblock In \emph{Advances in neural information processing systems 28}, pp.\
  1756--1764, 2015.

\bibitem[Baxter(2000)]{baxter2000model}
Baxter, J.
\newblock A model of inductive bias learning.
\newblock \emph{Journal of Artificial Intelligence Research}, 12:\penalty0
  149--198, 2000.

\bibitem[Ben-David \& Schuller(2003)Ben-David and Schuller]{ben2003exploiting}
Ben-David, S. and Schuller, R.
\newblock Exploiting task relatedness for multiple task learning.
\newblock In \emph{Learning Theory and Kernel Machines}, pp.\  567--580.
  Springer, 2003.

\bibitem[Bonawitz et~al.(2019)Bonawitz, Eichner, Grieskamp, Huba, Ingerman,
  Ivanov, Kiddon, Konecn{\'{y}}, Mazzocchi, McMahan, Overveldt, Petrou, Ramage,
  and Roselander]{bonawitz19sysml}
Bonawitz, K., Eichner, H., Grieskamp, W., Huba, D., Ingerman, A., Ivanov, V.,
  Kiddon, C., Konecn{\'{y}}, J., Mazzocchi, S., McMahan, H.~B., Overveldt,
  T.~V., Petrou, D., Ramage, D., and Roselander, J.
\newblock Towards federated learning at scale: System design.
\newblock In \emph{Conference on Systems and Machine Learning (SysML)}, 2019.
\newblock URL \url{http://arxiv.org/abs/1902.01046}.

\bibitem[Carmon et~al.(2017)Carmon, Duchi, Hinder, and
  Sidford]{carmon2017lower}
Carmon, Y., Duchi, J.~C., Hinder, O., and Sidford, A.
\newblock Lower bounds for finding stationary points i.
\newblock \emph{arXiv preprint arXiv:1710.11606}, 2017.

\bibitem[Caruana(1997)]{caruana1997multitask}
Caruana, R.
\newblock Multitask learning.
\newblock \emph{Machine learning}, 28\penalty0 (1):\penalty0 41--75, 1997.

\bibitem[Cesa-Bianchi et~al.(2007)Cesa-Bianchi, Mansour, and
  Stoltz]{Cesa-Bianchi2007}
Cesa-Bianchi, N., Mansour, Y., and Stoltz, G.
\newblock Improved second-order bounds for prediction with expert advice.
\newblock \emph{Machine Learning}, 66\penalty0 (2):\penalty0 321--352, Mar
  2007.
\newblock ISSN 1573-0565.
\newblock \doi{10.1007/s10994-006-5001-7}.
\newblock URL \url{https://doi.org/10.1007/s10994-006-5001-7}.

\bibitem[Even-Dar et~al.(2008)Even-Dar, Kearns, Mansour, and
  Wortman]{Even-Dar2008}
Even-Dar, E., Kearns, M., Mansour, Y., and Wortman, J.
\newblock Regret to the best vs. regret to the average.
\newblock \emph{Machine Learning}, 72\penalty0 (1):\penalty0 21--37, Aug 2008.
\newblock ISSN 1573-0565.
\newblock \doi{10.1007/s10994-008-5060-z}.
\newblock URL \url{https://doi.org/10.1007/s10994-008-5060-z}.

\bibitem[Evgeniou \& Pontil(2004)Evgeniou and Pontil]{evgeniou2004regularized}
Evgeniou, T. and Pontil, M.
\newblock Regularized multi--task learning.
\newblock In \emph{Proceedings of the tenth ACM SIGKDD international conference
  on Knowledge discovery and data mining}, pp.\  109--117. ACM, 2004.

\bibitem[Evgeniou et~al.(2005)Evgeniou, Micchelli, and
  Pontil]{evgeniou2005learning}
Evgeniou, T., Micchelli, C.~A., and Pontil, M.
\newblock Learning multiple tasks with kernel methods.
\newblock \emph{Journal of Machine Learning Research}, 6\penalty0
  (Apr):\penalty0 615--637, 2005.

\bibitem[Go et~al.(2009)Go, Bhayani, and Huang]{go09}
Go, A., Bhayani, R., and Huang, L.
\newblock Twitter sentiment classification using distant supervision.
\newblock \emph{CS224N Project Report, Stanford}, 150, 01 2009.
\newblock Data from \url{http://help.sentiment140.com/for-students}.

\bibitem[Hard et~al.(2018)Hard, Rao, Mathews, Beaufays, Augenstein, Eichner,
  Kiddon, and Ramage]{gboard}
Hard, A., Rao, K., Mathews, R., Beaufays, F., Augenstein, S., Eichner, H.,
  Kiddon, C., and Ramage, D.
\newblock Federated learning for mobile keyboard prediction.
\newblock \emph{arXiv preprint 1811.03604}, 2018.

\bibitem[Hazan(2016)]{hazan2016introduction}
Hazan, E.
\newblock Introduction to online convex optimization.
\newblock \emph{Foundations and Trends{\textregistered} in Optimization},
  2\penalty0 (3-4):\penalty0 157--325, 2016.

\bibitem[Kone{\v{c}}n{\'y} et~al.(2016)Kone{\v{c}}n{\'y}, McMahan, Ramage, and
  Richt{\'a}rik]{FL3}
Kone{\v{c}}n{\'y}, J., McMahan, H.~B., Ramage, D., and Richt{\'a}rik, P.
\newblock Federated optimization: Distributed machine learning for on-device
  intelligence.
\newblock \emph{arXiv preprint arXiv:1610.02527}, 2016.

\bibitem[Maurer(2006)]{maurer2006bounds}
Maurer, A.
\newblock Bounds for linear multi-task learning.
\newblock \emph{Journal of Machine Learning Research}, 7\penalty0
  (Jan):\penalty0 117--139, 2006.

\bibitem[McMahan \& Ramage(2017)McMahan and Ramage]{FL_BLOG}
McMahan, H.~B. and Ramage, D.
\newblock Federated learning: Collaborative machine learning without
  centralized training data, April 2017.
\newblock URL
  \url{https://ai.googleblog.com/2017/04/federated-learning-collaborative.html}.
\newblock Google AI Blog.

\bibitem[McMahan et~al.(2017)McMahan, Moore, Ramage, Hampson, and y~Arcas]{FL4}
McMahan, H.~B., Moore, E., Ramage, D., Hampson, S., and y~Arcas, B.~A.
\newblock Communication-efficient learning of deep networks from decentralized
  data.
\newblock In \emph{Proceedings of the 20th International Conference on
  Artificial Intelligence and Statistics}, pp.\  1273--1282, 2017.

\bibitem[Sani et~al.(2014)Sani, Neu, and Lazaric]{sani2014exploiting}
Sani, A., Neu, G., and Lazaric, A.
\newblock Exploiting easy data in online optimization.
\newblock In \emph{Advances in Neural Information Processing Systems}, pp.\
  810--818, 2014.

\bibitem[Shalev-Shwartz(2012)]{shalev2012online}
Shalev-Shwartz, S.
\newblock Online learning and online convex optimization.
\newblock \emph{Foundations and Trends{\textregistered} in Machine Learning},
  4\penalty0 (2):\penalty0 107--194, 2012.

\bibitem[Shalev-Shwartz et~al.(2009)Shalev-Shwartz, Shamir, Srebro, and
  Sridharan]{shalev2009stochastic}
Shalev-Shwartz, S., Shamir, O., Srebro, N., and Sridharan, K.
\newblock Stochastic convex optimization.
\newblock In \emph{Proceedings of the 22nd Annual Conference on Learning Theory
  (COLT)}, 2009.

\bibitem[Turlach et~al.(2005)Turlach, Venables, and
  Wright]{turlach2005simultaneous}
Turlach, B.~A., Venables, W.~N., and Wright, S.~J.
\newblock Simultaneous variable selection.
\newblock \emph{Technometrics}, 47\penalty0 (3):\penalty0 349--363, 2005.

\bibitem[Woodworth \& Srebro(2016)Woodworth and Srebro]{woodworth2016tight}
Woodworth, B. and Srebro, N.
\newblock Tight complexity bounds for optimizing composite objectives.
\newblock In \emph{Advances in neural information processing systems}, pp.\
  3639--3647, 2016.

\bibitem[Woodworth \& Srebro(2017)Woodworth and Srebro]{woodworth2017lower}
Woodworth, B. and Srebro, N.
\newblock Lower bound for randomized first order convex optimization.
\newblock \emph{arXiv preprint arXiv:1709.03594}, 2017.

\bibitem[Woodworth et~al.(2018)Woodworth, Wang, Smith, McMahan, and
  Srebro]{woodworth2018graph}
Woodworth, B., Wang, J., Smith, A., McMahan, B., and Srebro, N.
\newblock Graph oracle models, lower bounds, and gaps for parallel stochastic
  optimization.
\newblock In \emph{Advances in Neural Information Processing Systems 31}, pp.\
  8505--8515. 2018.

\bibitem[Zinkevich(2003)]{zinkevich2003online}
Zinkevich, M.
\newblock Online convex programming and generalized infinitesimal gradient
  ascent.
\newblock In \emph{Proceedings of the 20th International Conference on Machine
  Learning (ICML)}, pp.\  928--936, 2003.

\end{thebibliography}
\bibliographystyle{icml2019}

\clearpage
\appendix
\newtheorem{cor}{Corollary}

\section{Deferred Proofs from \cref{sec:hedge}}
\label{sec:hedge-proofs}

In this section, we give a proof of Lemma~\ref{lem:regret}. 
Such a result has been previously shown in~\citet{Even-Dar2008,sani2014exploiting}, building on a lemma of~\citet{Cesa-Bianchi2007}. 
We give a full proof for completeness. We start with a simple Lemma.

\begin{lemma} \label{lem:log}
For any $z > - \frac 1 2$, 
$$z - z^2 \leq \ln(1+z) \leq z.$$
\end{lemma}

\begin{proof}
The upper bound on $\ln(1+z)$ is standard, and follows e.g. by the concavity of the log function. For the lower bound, write
\begin{align*}
  f(z) = \ln(1+z) - (z - z^2).
\end{align*}
Then $f'(z) = \frac{1}{1+z} - 1 + 2z = \frac{z(1+2z)}{(1+z)}$. Thus $f$ is decreasing in $(-\frac 1 2, 0)$ and increasing in $(0,\infty)$. Thus in the range $(-\frac 1 2, \infty)$, $f(z) \geq f(0) = 0$.
\end{proof}

We consider the more general case of $K+1$ experts with losses in $[-M,M]$, and a chosen expert $0$, with respect to which we want constant regret. 
We consider the {\sc Prod} Algorithm that starts out with initial weights:
\begin{align*}
q^0_1 = 1-\eta; \,\,\,\,\,\,\, q^i_1 = \eta/K \,\,\,\,\,\,\forall i=1..K.
\end{align*}

At time step $t$, it picks an expert $j_t$ with probability proportional to $q^j_t$:
\begin{align*}
p^i_t &= q^i_{t} \big/ (\sum_{j=0}^K q^j_t).
\end{align*}
Finally, on receiving the loss function $\ell_t$, it updates the weights according to the multiplicative update
    	\begin{align*}
    		q^{i}_{t+1}
    		=
    		q^{i}_t \cdot \big(1+\eta\big(\ell_t(0) - \ell_t(i)\big)\big)
    		\qquad
    		\forall i=0..K
    	\end{align*}

\begin{lemma}
Assume that $0 < \eta \leq 1/(4M)$.
Then this {\sc Prod} algorithm achieves
\begin{align*}
  \sum_{t=1}^T p_t^{j_t} \ell_t(j_t) - \sum_{t=1}^T \ell_t(j)
  &\leq 4\eta M^2 T + \frac{1}{\eta} \ln \frac{K}{\eta}
  \intertext{for all $j=1,\ldots,K$, and}
  \sum_{t=1}^T p_t^{j_t} \ell_t(j_t) - \sum_{t=1}^T \ell_t(0)
  &\leq 1 + \eta
  .
\end{align*}
\end{lemma}

\begin{proof}
Let $Q_t = \sum_{j=0}^K q^j_t$ and let $\Delta^j_t = \ell_t(0) - \ell_t(j)$ denote the gap between the chosen expert and expert $j$ at step $t$. Note that $|\Delta^j_t| \leq 2M$.

On the one hand,
\begin{align*}
  \ln \frac{Q_{T+1}}{Q_1} &= \sum_{t=1}^{T} \ln \frac{Q_{t+1}}{Q_t}\\
  &= \sum_{t=1}^{T} \ln\Bigg( \frac{1}{Q_t} \sum_{j=0}^K q^j_t(1+\eta\Delta^j_t) \Bigg)\\
  &= \sum_{t=1}^T \ln \sum_{j=0}^K p^j_t(1+\eta\Delta^j_t)\\
  &= \sum_{t=1}^T \ln\Bigg( 1 + \eta \sum_{j=0}^K p^j_t\Delta^j_t \Bigg)\\
  &\leq \sum_{t=1}^T \eta \sum_{j=0}^K p^j_t\Delta^j_t\\
  &= \eta \sum_{t=1}^T \ell_t(0) - \eta \sum_{t=1}^T p_t^{j_t} \ell_t(j_t)
  .
\end{align*}
On the other hand, for any $j$,
\begin{align*}
    \ln \frac{Q_{T+1}}{Q_1} 
    &\geq \ln \frac{q^j_{T+1}}{q^j_1} + \ln \frac {q^j_1}{Q_1}\\
    &= \sum_{t=1}^T \ln \frac{q^j_{t+1}}{q^j_t} + \ln \frac {q^j_1}{Q_1} \\
    &= \sum_{t=1}^T \ln (1+ \eta\Delta^j_t) + \ln \frac {q^j_1}{Q_1} \\
    &\geq \sum_{t=1}^T (\eta\Delta^j_t - (\eta \Delta^j_t)^2) + \ln \frac {q^j_1}{Q_1}\\
    &\geq \eta \sum_{t=1}^T (\ell_t(0) - \ell_t(j)) - 4\eta^2 M^2 T + \ln \frac {q^j_1}{Q_1}
    ,
\end{align*}
where the middle inequality holds since $|\eta\Delta^j_t| \leq 1/2$.
It follows that for $j \neq 0$,
\begin{align*}
    \sum_{t=1}^T p_t^{j_t} \ell_t(j_t) - \sum_{t=1}^T \ell_t(j)
    \leq 
    4\eta M^2 T + \frac{1}{\eta} \ln \frac{K}{\eta}
    .
\end{align*}
Moreover, since $q^0_t$ does not change during the algorithm, we also have, using the lower bound in \cref{lem:log}, that
\begin{align*}
  \ln \frac{Q_{T+1}}{Q_1} 
  \geq 
  \ln \frac {q^0_1}{Q_1}
  =
  \ln(1-\eta)
  \geq
  -\eta-\eta^2
  .
\end{align*}
This implies that
\begin{align*}
    \sum_{t=1}^T p_t^{j_t} \ell_t(j_t) - \sum_{t=1}^T \ell_t(0)
    &\leq 
    1+\eta 
    .
    \qedhere
\end{align*}
\end{proof}

Optimizing parameters, we get the corollary:

\begin{cor}
Set $\eta = \frac{1}{2M} \sqrt{\ln(KMT)/T}$ and assume that that $M \geq 1$ and that $T$ is large enough so that $\eta \leq 1/(4M)$. 
Then the algorithm achieves the following regret bounds:
\begin{align*}
    \sum_{t=1}^T p_t^{j_t} \ell_t(j_t)
    &\leq \sum_{t=1}^T \ell_t(j) + 4M \sqrt{T\ln(KMT)}
    \intertext{for all $j=1,\ldots,K$, and}
    \sum_{t=1}^T p_t^{j_t} \ell_t(j_t)
    &\leq \sum_{t=1}^T \ell_t(0) + 2.
\end{align*}
\end{cor}

Lemma~\ref{lem:regret} follows from the $K=1$ version of this corollary, where the two experts are the algorithms $w^j_t$ and $w_t$.

\section{Experimental Details} \label{sec:exp_details}

\subsection{Dataset} The sentiment140 dataset set~\citep{go09} was collected by querying Twitter (a popular social network) for posts (a.k.a.~Tweets) containing positive and negative emoticons, and labeling the retrieved posts (with emoticons removed) as positive and negative sentiment, respectively. 

The data sets used for the above scenarios are created by first shuffling the data randomly and splitting it into a training (90\%, or $1,440,000$ examples) and test set (10\%, or $160,000$ examples). This data set is used as-is for training and evaluating the \textit{idealized i.i.d.} model. For the other scenarios trained on block-cyclic data, we group the shuffled training set and test set into $m=6$ blocks each by the time of day of the post (e.g. midnight block: posts from 12am - 4am; noon block: posts from 12pm - 4pm). This results in blocks of varying sizes, on average $1,440,000/6=240,000$ (training) and $160,000/6=24,000$ (testing) examples, respectively.

We simulate $K=10$ cycles (days). Observing that one pass (epoch) over the entire i.i.d.~data set was sufficient for convergence of our relatively small model, this results in $mn=1,440,000/10$ training examples per day, or $n=1,440,000/10/6=24,000$ training examples per day per block. 

\subsection{Artificially balanced labels} \label{sec:exp_labels}

\begin{figure}[t]
\centering
\includegraphics[width=8cm]{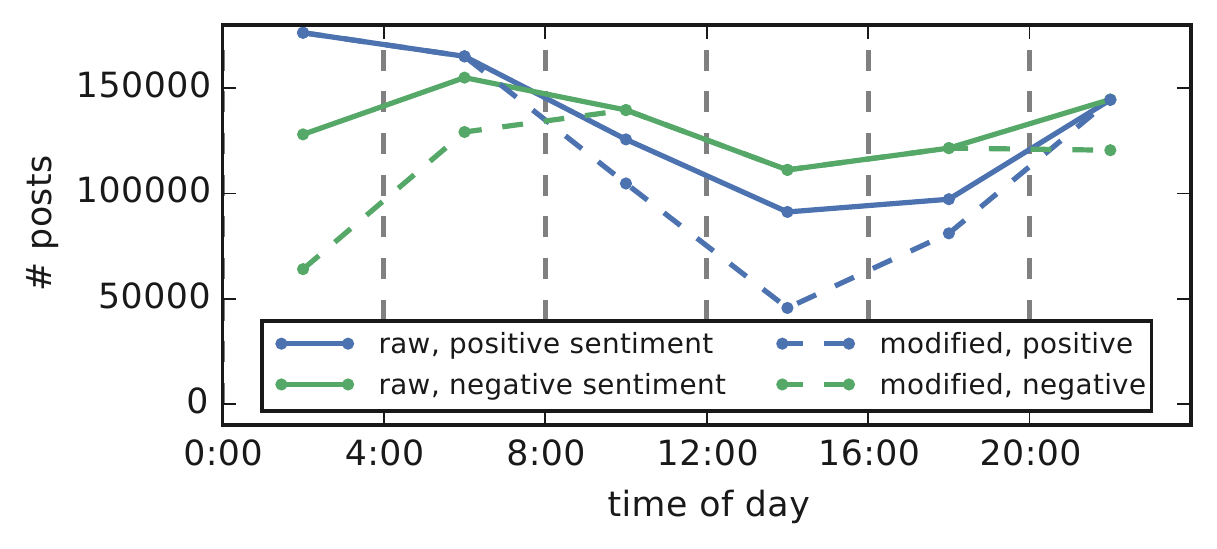}
\caption{Sentiment bias as a function of day time in the Sentiment140 dataset. For experiments in this paper, we introduced additional time-of-day dependent label skew to allow for a clearer illustration of how pluralistic approaches differ.}
\label{fig:sentiment140_bias}
\end{figure}

The raw data grouped by time of day exhibits some block-cyclic characteristics; for instance, positive tweets are slightly more likely at night time hours than day time hours (see Figure \ref{fig:sentiment140_bias}). However, we believe this dataset has an artificially balanced label distribution, which is not ideal to illustrate semi-cyclic behavior \cite{go09}. In particular, the data collection process separately queried the Twitter API every 2 minutes for positive tweets (defined to be those containing the \textbf{:)} emoticon), and simultaneously for negative sentiment via \textbf{:(}. Since only up to 100 results are returned via each API query, this will generally produce an (artificially) balanced label distribution, as in Fig.~\ref{fig:sentiment140_bias}. Due to this fact, because large diurnal variations are likely in practice in Federated Learning (e.g., differences in the use of English language between the US and India), and because it better illustrates our theoretical results, we adjust the positive-sentiment rate as a function of time as described in section \ref{sec:experiments}.

\subsection{Details of evaluation methodology}\label{sec:exp_eval}
For the \textit{block-cyclic consensus} model, picking a random iteration of the form $t(k, i, n)$ ensures we evaluate a set of models that have the same expected number of iterations as for the single-chain pluralistic approach, without using block-specific models. In the implementation, we compute the expectation of this quantity by evaluating all $m$ iterates against all $m$ $\hat{F}_i$s, and averaging these $m^2$ values.

This same $m \times m$ set of evaluation results is used to evaluate the pluralistic single SGD chain approach, but instead of averaging all $m^2$ accuracies, we only consider the diagonal, where the model most recently trained on data from component $i$ is evaluated (only) on $\hat{F}_i$.

\end{document}